\newif\ifcomments  
\commentsfalse
\documentclass{article}

\usepackage{amsmath,amssymb,pifont}
\usepackage{multicol}
\usepackage{amstext}
\usepackage{amsthm}
\usepackage{multirow}
\usepackage{booktabs}
\usepackage[skip=0pt]{subcaption}
\usepackage{times}
\usepackage{lipsum}
\usepackage[shortlabels]{enumitem}
\usepackage{cancel}
\usepackage{wrapfig}
\usepackage{array}
\usepackage{siunitx}
\usepackage{csvsimple}
\usepackage[multidot]{grffile}
\usepackage{bbm}
\usepackage{dblfloatfix}
\usepackage{hyperref}
\usepackage{makecell}
\usepackage{bbm, dsfont}
\usepackage{mathtools}
\usepackage{xcolor}
\usepackage{comment}
\usepackage{algcompatible}
\usepackage[multiple]{footmisc}
\usepackage{mathrsfs}
\usepackage{tikz}
\usepackage{pdflscape}
\usepackage{float}

\usepackage[ruled,vlined]{algorithm2e}

\hypersetup{final}

\newcommand{\boldu}{\ensuremath{\boldsymbol{u}}}

\newcommand{\boldx}{\bfx}

\newcommand{\boldG}{\ensuremath{\boldsymbol{G}}}
\newcommand{\boldH}{\ensuremath{\boldsymbol{H}}}
\newcommand{\boldI}{\ensuremath{\boldsymbol{I}}}

\newcommand{\boldR}{\bfR}

\newcommand{\bfR}{\ensuremath{\mathbf{R}}}

\newcommand{\bfx}{\ensuremath{\mathbf{x}}}

\newcommand{\calA}{\ensuremath{\mathcal{A}}}

\newcommand{\calC}{\ensuremath{\mathcal{C}}}

\newcommand{\calF}{\ensuremath{\mathcal{F}}}

\newcommand{\calH}{\ensuremath{\mathcal{H}}}

\newcommand{\calK}{\ensuremath{\mathcal{K}}}
\newcommand{\calL}{\ensuremath{\mathcal{L}}}

\newcommand{\calN}{\ensuremath{\mathcal{N}}}

\newcommand{\calP}{\ensuremath{\mathcal{P}}}

\newcommand{\calS}{\ensuremath{\mathcal{S}}}

\newcommand{\calX}{\ensuremath{\mathcal{X}}}

\renewcommand{\Pr}{\mathop{\mathbf{Pr}}}

\newtheorem{lem}{Lemma}[section]

\newtheorem{thm}[lem]{Theorem}
\newtheorem{cor}[lem]{Corollary}

\newtheorem{definition}[lem]{Definition}
\newtheorem{remark}{Remark}

\DeclareMathOperator*{\argmin}{arg\,min}

\makeatletter
\newcommand{\vast}{\bBigg@{4}}
\newcommand{\Vast}{\bBigg@{5}}
\makeatother

\newcommand{\privT}{\theta^\texttt{priv}\,}

\newcommand{\ltwo}[1]{\left\|#1\right\|_2}

\DeclarePairedDelimiterX{\infdivx}[2]{(}{)}{%
  #1\;\delimsize\|\;#2%
}

\newcommand{\eref}[1]{{Eq}~\ref{#1}}

\newcommand{\aref}[1]{{Algorithm}~\ref{#1}}

\newcommand{\regret}[2]{{\sf Regret}_T(#1;#2)\,}

\newcommand{\mypar}[1]{\smallskip
	\noindent{\textbf{{#1}:}}}
	
\renewcommand{\epsilon}{\varepsilon}
\newcommand{\nadagrad}{\calA_{\sf noisy-AdaGrad}\,}

\newcommand{\dpgd}{\calA_{\sf DP-GD}\,}
\newcommand{\tr}{\mathbf{Tr}\,}
\newcommand{\rank}{\text{rank}\,}

\newcommand{\apriv}{\calA_{\sf priv}\,}
\newcommand{\risk}[1]{{\sf Risk}(#1)\,}
\newcommand{\ptheta}{\theta_{\sf priv}\,}
\renewcommand{\tilde}{\widetilde}

\usepackage{fullpage}
\usepackage[numbers, sort, comma, square]{natbib}

\title{Fast Dimension Independent Private AdaGrad on Publicly Estimated Subspaces}
\author{Peter Kairouz
\thanks{Google Research. \texttt{\{kairouz, mribero, krush\}@google.com}} 
\and 
M\'onica Ribero\footnotemark[2] \thanks{Department of Electrical and Computer Engineering, The University of Texas at Austin. Work done while author was interning at Google.}\and
Keith Rush \footnotemark[2]
\and
Abhradeep Thakurta\thanks{Google Research - Brain.  \texttt{\{athakurta\}@google.com}}}

\usepackage{times}

\begin{document}

\maketitle

\begin{abstract}%
  We revisit the problem of empirical risk minimziation (ERM) with differential privacy. We show that noisy AdaGrad, given appropriate knowledge and conditions on the subspace from which gradients can be drawn, achieves a regret comparable to traditional AdaGrad plus a well-controlled term due to noise. We show a convergence rate of $O(\tr(G_T)/T)$, where $G_T$ captures the geometry of the gradient subspace. Since $\tr(G_T)=O(\sqrt{T})$ we can obtain faster rates for convex and Lipschitz functions, compared to the $O(1/\sqrt{T})$ rate achieved by known versions of noisy (stochastic) gradient descent with comparable noise variance. In particular, we show that if the gradients lie in a known constant rank subspace, and assuming algorithmic access to an envelope which bounds decaying sensitivity, one can achieve faster convergence to an excess empirical risk of $\tilde O(1/\epsilon n)$, where $\epsilon$ is the privacy budget and $n$ the number of samples. Letting  $p$ be the problem dimension, this result implies that, by running noisy Adagrad, we can bypass the DP-SGD bound $\tilde O(\sqrt{p}/\epsilon n)$ in $T=(\epsilon n)^{2/(1+2\alpha)}$ iterations, where $\alpha \geq 0$ is a parameter controlling gradient norm decay, instead of the rate achieved by SGD of $T=\epsilon^2n^2$. Our results operate with general convex functions in both constrained and unconstrained minimization. 

Along the way, we do a perturbation analysis of noisy AdaGrad of independent interest. Our utility guarantee for the private ERM problem follows as a corollary to the regret guarantee of noisy AdaGrad. 
\end{abstract}

\section{Introduction}
\label{sec:intro}

Differentially private convex optimization is a fundamental problem for machine learning practitioners. Empirical Risk Minimization (ERM) in particular is foundational in most learning tasks, many of which are posed over datasets with sensitive information that can be leaked through model parameters~\cite{fredrikson2015model, wu2016methodology, shokri2017membership}.  Differential privacy \cite{DMNS,ODO} has therefore been adopted in optimization when training machine learning models to limit user data exposure.

In current applications, models are usually many times over-parametrized. This is a major problem for private settings, where the the optimal model $\theta^*$ cannot be released, but we must release rather a private model $\ptheta$. For a model dimensionality of $p$, a naive privatization incurs an excess empirical risk with \textit{lower bound} linear in $\sqrt{p}$ \citep{BST14}.

In this paper we propose noisy-AdaGrad, a novel optimization algorithm that leverages gradient pre-conditioning and knowledge of the subspace in which gradients lie to recover AdaGrad regret rates ( $O\left(\frac{\tr(G_T)}{T}\right)$ where $G_T$ is the adaptive pre-conditioner defined in Equation~\ref{eq:preconditioner_update} in Algorithm~\ref{alg:noisy_adagrad_public_subspace}), and  dimension independent excess risk bounds. We propose a general framework to study noisy versions of Adaptive Pre-conditioning  (a.k.a. AdaGrad ~\citep{mcmahan10boundopt, duchi2011adaptive,hazan2019introduction}). Further, our analysis identifies a simple condition under which AdaGrad-style rates can be achieved in the differentially-private ERM problem: that of oracle access to a constant-factor envelope of the maximum gradient norm across data samples as training progresses (See Definition~\ref{def:envelope}).

In a concurrent and independent work,  \cite{zhou2020bypassing} also found dimension independent bounds through the analysis of a projected version of stochastic gradient descent (PDP-SGD). Our work differs significantly in the regret analysis of our algorithm; the matrix-perturbation analysis presented here can be of independent interest. Furthermore, the analysis we perform is vital to highlighting the way towards recovery of AdaGrad-style rates in the private setting, opening an area of future development for accelerating differentially-private optimization.

Each of the assumed pieces of input data is well-justified in practice. Knowledge about the gradient subspace is often available through public data that is easily accessible, for instance through ``opt-in" users  \cite{beimel2013private,xin2014controlling, alon2019limits, zhou2020bypassing}. For example, for Generalized Linear Models (GLM's) this subspace corresponds to the feature space determined by the column space of the data matrix (see for example \cite{song2020characterizing}). Knowledge of the maximum gradient norm can be had by observing the training procedure, and gradient norms for many classes of well-studied problems decay uniformly for all data samples, e.g. those studied in \cite{bassily2018exponential, ma2018power}. We leave it as an open problem to design a differentially private algorithm for computing this envelope.

\subsection{Problem Definition}
\label{sec:problem_def}

Let $D=\{d_1,\ldots,d_n\}$ be a given data set drawn from a distribution $\calP$, $\ell(\cdot, d)$ a map defining the loss on data point $d$, and an objective function $\calL(\theta;D)=\frac{1}{n}\sum\limits_{i=1}^n \ell(\theta;d_i)$. The goal is to design an $(\epsilon,\delta)$- differentially private algorithm $\apriv$ that outputs a model $\ptheta \in \calC \subseteq \boldR^p$ that approximately solves the following optimization problem: 

\begin{equation}
    \label{eq:problem_def}
    \min\limits_{\theta \in \calC}\calL(\theta;D).
\end{equation} In terms of accuracy we consider the traditional excess empirical risk defined as follows:
\begin{equation}
    \risk{\ptheta}=\calL(\ptheta;D)-\min\limits_{\theta}\calL(\theta;D).
    \label{eq:risk}
\end{equation}

Being consistent with the literature on private convex ERM, we will assume each of the loss functions $\ell(\theta;d)$ is convex and $L$-Lipschitz in its first parameter w.r.t. the $\ell_2$-norm.

\mypar{Online convex optimization} To solve the private ERM problem, we will model it along the lines of online convex optimization~\citep{hazan2019introduction,shalev2011online}. First, we will propose a noise-tolerant algorithm for the traditional online convex optimization, and then use that algorithm and its analysis to design a differentially private ERM algorithm with a bound on the excess empirical risk. We use the well-known \emph{online to batch conversion}~\citep{hazan2019introduction} to translate the regret guarantee for an online algorithm to that of excess empirical risk of a convex optimization problem.

We adhere to the standard regret minimization setting of traditional online learning~\citep{hazan2019introduction}. Formally, given a sequence of loss functions $
\calF=\{f_1,\ldots,f_T\}$ (with each $f_t:\boldR^p\to\boldR$ ) arriving online, the objective is to design an algorithm to ouput a sequence of models $\{\theta_1,\ldots,\theta_T\}$ s.t. the following is minimized:
\begin{equation}
    \regret{\calF}{\calA}=\frac{1}{T}\sum\limits_{t=1}^T f_t(\theta_t)-\min\limits_{\theta} \frac{1}{T}\sum\limits_{t=1}^T f_t(\theta)
    \label{eq:regret}
\end{equation}
Throughout this paper, we will call an algorithm $\calA$ to be a ``low-regret'' algorithm if it outputs a sequence of models s.t. the regret in \eqref{eq:regret} is $o(1)$. In principle each of the the loss functions $f_t\in\calF$ can be chosen adaptively (and adversarially) based on the models output so far, i.e., $\theta_1,\ldots,\theta_{t-1}$. In this paper we will primarily focus on the \emph{convex setting}, where the loss functions in $\calF$ are assumed to be convex in its first parameter. Furthermore, we will assume that the loss functions are Lipschitz bounded, i.e., $\forall \theta\in\boldR^p,f\in\calF:\ltwo{\partial_\theta f(\theta)}\leq L$.

\subsection{Our Contributions}
\label{sec:intro_results}
Our main contribution is to obtain dimension independent excess risk bounds for differentially private ERM through adaptive pre-conditioning. Our contributions can be stated as follows. 

\mypar{A noise tolerant AdaGrad-style algorithm} We design Noisy-AdaGrad ($\nadagrad$, Algorithm ~\ref{alg:noisy_adagrad_public_subspace}), a novel noise tolerant optimization algorithm with adaptive preconditioning that, under appropriate parameter selection, satisfies $(\epsilon, \delta) - $ differential privacy. The algorithm differs from AdaGrad in three main respects: (1) it uses a gradient perturbed with Gaussian noise; (2) the pre-conditioner is updated with clean gradients and then perturbed with a noise matrix drawn from the Gaussian Orthogonal Ensemble; (3) we introduce a projection step that is intended to maintain the trajectory of the descent algorithm in the gradients' subspace. We assume (noisy) oracle access $\tilde{V}_t$ to $V_t$, the orthogonal matrix whose columns span the gradient subspace at iteration $t$, and before taking a gradient step we project the update step using $\tilde{V}_t\tilde{V}_t^T$. In over-parameterized regimes, this step allows us to significantly decrease the effect of noise, when gradients lie in a low rank subspace, a common characteristic in high-dimensional problems \citep{agarwal2018efficient, gur2018gradient}. In practice, this subspace can be computed from public data \cite{beimel2013private, alon2019limits, zhou2020bypassing}. 

\mypar{Dimension independent and AdaGrad-style regret rates with noisy gradient subspace}
 We provide a dimension independent low regret bound for $\nadagrad$ in Theorem~\ref{thm:noisy_ada_noisy_public_subspace} that recovers AdaGrad rates given access to a simple sensitivity oracle, improving over previous gradient descent rates. This is, to the best of our knowledge, the first work to analyse a noisy version of full matrix AdaGrad where both the gradient and pre-conditioner are independently noised. 
    Our main regret bound is the following. 
        \begin{thm}[Informal version of Theorem~\ref{thm:noisy_ada_noisy_public_subspace}]
            \label{thm:noisy_ada_regret_informal}
            Let $V_t$ be the orthogonal matrix whose column space is the tracked gradient subspace up to time $t$, and $\tilde{V_t}$ an approximation returned by an oracle. Let $\gamma$ be a bound on the subspaces' principal angle difference, i.e., $\|V_tV_t^T -\tilde{V}_t\tilde{V}_t^T  \|_{op} \leq \gamma$. Let $L$ be the gradient $\ell_2-$norm bound, $C$ the diameter of the constraint set $\calC$, and assume $L=C=O(1)$. Letting $\sigma^2_b(t)$ be the gradient noise variance, and choosing the pre-conditioner noise appropriately then running $\nadagrad$ on $\calL(\theta;D)$ for $T$ iterations we get 
            \begin{flalign*}
                \mathbb{E}[ \regret{\calF}{\nadagrad}] &&
            \end{flalign*}
            \begin{equation}
            \label{eq:noisy_ada_public_data}
                 \leq O \left( \mathbb{E}\left[  \sqrt{  \regret{\calF}{\calA_{\sf AdaGrad}}^2 + \frac{ \tr(G_T)\sum_{t=1}^{T}\sigma^2_b(t) \tr(G_t^{-1})}{T^2}} + \gamma\right] \right)
            \end{equation}
        \end{thm}
    Our result can be interpreted as follows. 
    \begin{itemize}
        \item The AdaGrad regret term in our bound reduces to $O(\tr(G_T)/T)$, improving over SGD which achieves regret $O(1/\sqrt T)$.
        \item The second term only depends on the gradient space dimension, dictated by the clean pre-conditioner $G_t$, unlike DP-SGD where this term linear in $\sqrt{p}$. By incorporating a projection to gradient subspace, we obtain dimension independence. Furthermore, we show in Corollary~\ref{cor:smooth_convergence} that by adapting the gradient noise at each iteration to be similar in scale to the gradient, we obtain faster rates: again,  $O(\tr(G_T)/T)$. 
        \item An additive factor $\gamma$ accounting for subspace estimation mismatch. We use Davis-Kahan $\sin(\theta)$ theorem to bound errors due to rotation of the problem space.  
    \end{itemize}

    This analysis can be of independent interest and is crucial for any differential privacy guarantee, since the pre-conditioner used for AdaGrad contains the full history of gradients. 
    In practice, the alternative to private AdaGrad has been to update the pre-conditioner with noisy gradients and rely on the post-processing property of differential privacy. 
    
    \mypar{Dimension independent excess empirical risk bounds for private AdaGrad with public data}
Our third result is to derive an excess risk bound that addresses the case where noise parameters are set to provide differential privacy. Our algorithm uses public data to compute the projection matrix $\tilde{V}_t\tilde{V}_t^T$ that forces the descent algorithm to stay in the gradient subspace, and the analysis derives a dimension independent excess risk bound for differentially private AdaGrad. Setting  gradient and pre-conditioner noise variances appropriately, $\nadagrad$ is differentially private and we  obtain an excess risk of $\frac{1}{\epsilon n}$ in $T = (\epsilon n)^{2/(1+2 \alpha)}$, where $\alpha$ controls the decay rate of gradients norm. This means that if $\alpha>0$ we reach the excess risk faster than (P)DP-SGD, that has running time  $T=\epsilon^2 n^2$. Additionally, we include in Lemma~\ref{lem:preconditioner_sensitivity} the non-trivial computation of the pre-conditioner's sensitivity.

    \begin{cor}[Informal version of Corollary~\ref{cor:private_ada}]
    Given a minimization problem where the subspace spanned by gradients has bounded rank $k<p$, running $\nadagrad$ with appropriate noise parameters is $(\epsilon, \delta)$- differentially private and the expected excess risk of $\nadagrad$ is $O\left(\frac{\sqrt{\log(1/\delta)}}{\epsilon n}\right) $.
    \end{cor}

\mypar{Dimension independent excess empirical risk for DP-SGD without public data}
Finally, we extend previous results from \cite{song2020characterizing} and show that for unconstrained minimization, DP-SGD, without any gradient subspace knowledge, is enough to obtain an excess risk bound of $O\left(\frac{\sqrt{\log(1/\delta)}}{\epsilon n}\right)$ independent of dimension. 
        \begin{thm}[Informal version of Theorem~\ref{thm:unconstrained_dpgd}]
        Let $\theta_0 = \bold0$ be the initial point of $\dpgd$. Let $\theta^* = \argmin\limits_{\theta \in \mathbb{R}^p}$ and $M=VV^T$ be the projector to the gradient eigenspace. Letting $L$ be the gradient $\ell_2-$norm bound, setting the constraint set $\calC = \mathbb{R}^p$, and running $\dpgd$ on $\calL(\theta;D)$ for $T=\epsilon^2n^2$ and appropriate learning rate $\eta$, 
        \begin{equation}
            \label{eq:unconstrained_dpgd}
            \mathbb{E}[\calL(\ptheta;D)]-\calL(\theta^*;D)\leq \frac{L\|\theta^*\|_{M}\sqrt{1+2\text{rank}(M)\log(1/\delta)}}{\epsilon n}
        \end{equation}
        
        \end{thm}


The concurrent work of \cite{zhou2020bypassing} studies a similar problem, incorporating gaussian noise to privatize the gradient and a publicly available projection to gradient subspace to achieve dimension independence for differentially private SGD. The methods of proof, however, are significantly different, and the results presented here conditionally achieve faster convergence.

\subsection{Techniques}
\label{sec:intro_techniques}
In this section we describe the main techniques leveraged to obtain the above results. Our contributions are structured as follows: we first analyse the regret of noisy-Adagrad, introduced in Algorithm~\ref{alg:noisy_adagrad_public_subspace}. Second, we use this analysis to provide excess risk bounds. Third, and finally, we translate these results for the case when the noise in $\nadagrad$ is intended to provide privacy.

\paragraph{Noisy-Adagrad.} The first part of the proof of  Theorem~\ref{thm:noisy_ada_noisy_public_subspace} bounds the regret of our noisy-AdaGrad algorithm, relying on matrix perturbation analysis. The proof follows standard convexity arguments to bound the regret with a linear approximation, resulting in the four terms in Equation~\ref{eq:results_regret_terms}. Although this expression is analogous to the original AdaGrad regret bound, the analysis in our case is much more involved due to gradient noise $b_t$ and pre-conditioner noise $B_t$. Given that we need our bound in terms of the original pre-conditioner $G_t$, we introduce several findings and key lemmas that allow us to achieve this. We summarize them below.

Equation~\ref{eq:results_regret_terms} is composed of four terms that can be independently bounded: a potential drop term that captures closeness to the optimum,  a gradient noise norm term, a gradient norm term, and a projection error term. 

\begin{align}
\label{eq:results_regret_terms}  f\left(\frac{1}{T} \sum_t \theta_t\right) & - f(\theta^*)  \leq  \sum_t  \frac{1}{2\eta T } \left(  \| \theta_t-\theta^*\|^2_{H_t} - \|\theta_{t+1}-\theta^* \|^2_{H_t} \right) \\ \nonumber &+ \frac{\eta}{2T} \left(
      \mathbb{E}_{b_t}[ \|b_t\|_{H_t^{-1}}^2 | B_t]\right)   + \frac{\eta}{2T} \left(\|\nabla_t \|_{H_t^{-1}}^2\right)  + \frac{1}{T} \langle \nabla_t, \theta_t - \theta^* \rangle|_{C_t}
\end{align}

We briefly describe the additional difficulties of analyzing our algorithm compared to traditional AdaGrad or DP-SGD.

The first term, the potential drop involving matrix norms $\|\cdot\|_{H_t}$,  is traditionally bounded using a telescoping argument, resulting in a $\tr(G_t)$ term. Here we first need to manipulate this expression and rely on trace definition and properties, and the fact that $B_t$ is zero mean to obtain a similar result in terms of $G_t$ and not $H_t$.

To bound the second and third terms involving matrix norm $\|\cdot\|_{H^{-1}_t}$, we first introduce a high-probability bound on the pre-conditioner noise matrix operator norm, $\|B_t \|_{op}$, and then use it to prove structural Lemma~\ref{lem:h-inverse}. This Lemma uses the Woodbury identity to calculate the inverse of a sum of matrices (in this case the pre-conditioner matrix $G_t$, and the pre-conditioner noise $B_t$).

The last term makes use of the Davis-Kahan theorem (Theorem~\ref{thm:davis-kahan}) to bound the principal angle difference between two subspaces: this allows to measure how much signal is lost by projecting onto a perturbed subspace.

\paragraph{Excess Empirical Risk.} It is a well-known standard idea called \emph{online to batch conversion}~\citep{cesa2004generalization,SSSS09} to translate the regret guarantee for an online algorithm to that of excess empirical risk of a convex optimization problem.

\paragraph{Providing Privacy.} To set noise values, we compute the $\ell_2$ sensitivity of gradients and pre-conditioner. Since individual data point loss functions $\ell$ are assumed $L-$Lipschitz, the sensitivity of the overall loss function's gradients can be bounded by $\frac{L}{n}$. The pre-conditioner sensitivity is more involved, since at each iteration $t$, it utilizes the full history of gradients. We show in Lemma~\ref{lem:preconditioner_sensitivity} that it can be bounded by $\frac{\sqrt{T}L}{n}$. To the best of our knowledge, this is the first time the $\ell_2$ sensitivity of the pre-conditioner is explicitly computed; previous private Adagrad results relied on the post-processing property of DP, and used private gradients to update the pre-conditioner. This easy fix turns out to be inefficient since it adds bias to the pre-conditioner, slowing down the exploration advantage (large learning rates in unexplored directions) of the original AdaGrad algorithm.

Finally, relying on the Gaussian mechanism and strong composition of differential privacy, we show $\nadagrad$ can be adapted for privacy and achieve an excess risk of $\frac{1}{\epsilon n}$. 

Using standard techniques, one can work with an $\ell_2-$norm regularized loss and derive excess population risk guarantees (see Theorem 2 in \cite{shalev2009stochastic}) .  

\subsection{Other Related Work}
\label{sec:prior_work}
Differentially private ERM has been widely studied theoretically and empirically ~\citep{chaudhuri2011differentially,BST14,song2013stochastic,DP-DL,BassilyFTT19,mcmahan2017learning, WLKCJN17, iyengar2019towards, pichapati2019adaclip,TAB19,feldman2019private,song2020characterizing}. It was established by \cite{BST14} that the excess risk in the constrained setting for any differentially private optimization algorithm over convex functions is lower bounded by $\Omega\left(\frac{\sqrt{p}}{\epsilon n}\right)$. \cite{song2020characterizing} show that it is possible in the unconstrained setting to obtain a dimension independenr bound for Generalized Linear Models (GLM's). 

Some work has explored the non-convex setting. \cite{zhou2020bypassing} uses public data and DP-SGD to obtain logarithmic dependence on the dimension. Noisy versions of AdaGrad where the pre-conditioner is updated with information from noisy gradients have also been studied \citep{xie2020linear, zhou2020private, zhou2020towards}.

When using pre-conditioning we aim at working in the intrinsic subspace of the problem; therefore our work is also tangential to differentially private and noisy subspace estimation, which have been a broad area of study \citep{dwork2014analyze, upadhyay2020framework}.

\subsection{Notation}
We use $\| \cdot \|_2$ to denote the $\ell_2$ norm of a vector. We denote by $\lambda_i(A)$ the $i$-th largest eigenvalue of matrix $A$, $\lambda_{_{\min>0}}(A)$ the smallest positive eigenvalue of $A$, and $\|\cdot\|_{op}$ to denote the operator norm of a matrix, defined as $\|A\|_{op} = \max\{|\lambda_i| : \lambda_i \text{ eigenvalue value of A } \}$.
$\| \cdot \|_A$ denotes the Mahalanobis seminorm defined as $\| \cdot \|_A = \sqrt{\langle\cdot,A \cdot\rangle}$ for $A$ symmetric and positive-semidefinite. The dual norm to a norm $\| \cdot\|$ is defined as $\|x \|^* = \sup_{y:\|y\|\leq 1}\langle x,y\rangle$. The dual norm of the above matrix norm is given by $\|x \|^*_A = \|x\|_{A^{-1}}$.
We use $[T]$ to denote the time interval $[T] = \{ 1, ..., T\}$. Finally, in the considered setting, $f_t$ will be constant over time, so we will denote $f_t = f$, and to simplify notation we use $\nabla_t$ to denote $\nabla f(\theta_t)$.

\section{Background}

In this section we introduce the necessary tools for the analysis of our Noisy-AdaGrad algorithm. We start by introducing the traditional AdaGrad algorithm, followed by standard differential privacy definitions. 

\paragraph{AdaGrad. }
AdaGrad (Adaptive Gradient Descent) \citep{duchi2011adaptive,mcmahan10boundopt, hazan2019introduction} achieves low-regret for convex loss functions. One of the main features that separates AdaGrad from other online convex optimization algorithms like follow-the-regularized-leader,  online gradient descent~\citep{hazan2007logarithmic}, and online mirror descent~\citep{ben2001lectures,SSSS09} is the use of a gradient pre-conditioner. It allows much tighter regret guarantees if the gradients of the loss functions come from a constant (close to) low-rank subspace.

The original AdaGrad algorithm (Appendix~\ref{sec:appendix_adagrad}) proposes the following update with a convex constraint set $\calC$

\begin{equation}
    \theta_{t+1} = \argmin_{\theta \in \calC} \| \theta - (\theta_t - \eta\boldG_t^{-1}\nabla_t) \|^2_{G_t},
\end{equation}

AdaGrad is derived by analyzing the optimal (strongly convex) regularization function to use in hindsight, that would minimize the regret of an online convex optimization algorithm. Concretely, consider the set of all strongly convex regularization functions with a fixed and bounded Hessian in the set

\begin{equation}
    \mathcal{H} = \{ X \in \mathbb{R}^{p\times p} : \textbf{Tr}(X)\leq 1, X \succeq 0 \}
\end{equation}

AdaGrad achieves a regret bound that is within a constant factor of $2C$ of the regret achieved by the best, fixed pre-conditioner in hindsight. We formalize this in theorem \ref{thm:regret_adagrad}

\begin{thm}{(Theorem 5.11. in \cite{hazan2019introduction}, originally Theorem 6 in \cite{mcmahan10boundopt} and Theorem 8 in \cite{duchi2011adaptive}) }
\label{thm:regret_adagrad}
Let $\{\boldx_t \}$ be defined by \aref{alg:adagrad} with parameters $\eta=C$, where 
$C = \max_{\boldu \in \calK} \|\boldu - \boldx \|_2$

Then for any $\boldx^* \in \calK$, 
\begin{equation}
    \regret{\calF}{\calA_{\sf Adagrad}} \leq 2C\sqrt{\min_{H \in \calH} \sum_t \| \nabla_t \|_{H}^{*2}}
\end{equation}
\end{thm}


\paragraph{Differential Privacy. } Originally proposed by \cite{ODO,DMNS}, differential privacy is a framework protecting single records in a database by bounding the probability of re-identifying any record from a query output. In this paper we limit ourselves to approximate differential privacy, and we rely on the Gaussian mechanism and Renyi composition theorem to provide these privacy guarantees (see Appendix~\ref{appendix:dp}). Formally, 

\begin{definition}[(approximate) Differential Privacy \citep{ODO,DMNS}]
A randomized algorithm $\calA$ that receives as input a dataset $D$ is $(\epsilon, \delta)-$ differentially private if, for any pair of neighboring  datasets $D$ and $D'$( Definition~\ref{def:neighbor_datasets}), and any set of events $\calS$ in the range of $\calA$, 

$$ \mathbf{Pr}(\calA(D) \in \calS) \leq e^{\epsilon} \mathbf{Pr}(\calA(D') \in \calS)  + \delta,$$
where the probability is taken over the random coins of $\calA$.
\end{definition}

\section{Analysis of Noisy-Adagrad}
\label{sec:noisy_ada}
 In this section we present and study a noisy version of AdaGrad (see Algorithm \ref{alg:noisy_adagrad_public_subspace}), where the adaptive pre-conditioner is perturbed with a matrix sampled from the Gaussian Orthonormal Ensemble (GOE)
(Definition~\ref{def:GOE}), and the observed gradients are perturbed with spherical Gaussian noise. Assuming that gradients of the loss function along the trajectory of the models output by noisy AdaGrad lie in an accessible constant rank subspace, and an oracle providing an asymptotically correct estimate of the maximum gradient across data samples, we show: \emph{Asymptotically, the regret of noisy AdaGrad is within a constant factor of traditional AdaGrad.}

Formally, let $V$ be an orthonormal matrix whose columns span the gradient subspace. It is shown in \cite{song2020characterizing} that for generalized linear models (GLM's), unconstrained DP-(S)GD achieves a dimension independent bound. The proof relies on restricting the analysis to the feature subspace, which corresponds for these problems to the gradient subspace spanned by the columns of $V$. Even though the algorithm is oblivious to $V$, by tracking the error only in this region, in expectation the error is dimension independent.  We extend this result to constrained optimization, by introducing Algorithm \ref{alg:noisy_adagrad_public_subspace} that utilizes $V_t$, the matrix whose columns span the gradient subspace up to time $t$, to achieve dimension-independent bounds for this constrained setting. In practice it is highly unlikely we can compute the true subspace, but it is often the case that we have (noisy) oracle access to the subspace. For example, when there is public data available, it is possible to compute a noisy version $\tilde{V}_t$ of $V_t$. In \ref{thm:noisy_ada_noisy_public_subspace} we prove we can still obtain dimension independence with an extra factor of $\gamma$ that accounts for the distribution difference between the real subspace and the one obtained from the oracle.

\begin{algorithm}[bt!]
\SetAlgoLined
\KwIn{Learning rate $\eta$, $\theta_0 \in \boldR^p$,
 Gradient noise standard deviation $\sigma_b(t)$, GOE scaling $\sigma_B(t)$, oracle access to $\tilde{V}_t$ estimate of $V_t$  for $t \in [T]$, $S_0 \gets \mathbf{0}$}

 \For{t=1 \text{to} T}{
  Predict $\theta_t$, suffer loss $f(\theta_t)$ \;
  Update
      \begin{align}
            \label{eq:preconditioner_update}  S_t & = S_{t-1} + \nabla_t\nabla_t^T, \quad G_t = S_t^{1/2} , \quad B_t = \sigma_B(t) M_p \quad \text{ and } \quad M_p \sim \mu_{GOE}\;\\
            \nonumber  \tilde{V}_t & \gets \text{Gradient subspace returned by the oracle}\; \\
            \label{eq:preconditioner_projection} H_t   &= \Pi_t(G_t + B_t) \quad \text{ where } \Pi_t = \tilde{V}_t\tilde{V}_t^T \\
            \nonumber \tilde{\nabla_t} &= \nabla_t+b_t \quad \text{ where } b_t \sim \mathcal{N}(0,\sigma_b^2(t) I_p)
        \end{align}
    \newline Denote $k_t  =\text{rank}(H_t)$.
    \begin{align}
           \nonumber y_{t+1} &= \theta_{t} - \eta H_t^{-1}\tilde{\nabla}_t \\
           \label{eq:adagrad_noisy_update} \theta_{t+1}  & \in  P_{\calC}^{H_t}(y_{t+1})
      \end{align}
    where $P_{\cal_C}^{H}(y) =  \argmin_{\theta \in \calC}\| \theta - y \|_{H}$ denotes the projection over the convex set $\calC$ using the semi-norm determined by $H$.
 }
 \KwResult{$ \{ \theta_t \}_{t=1}^T$}
 \caption{Noisy Adagrad ($\nadagrad$) with gradient subspace oracle}
 \label{alg:noisy_adagrad_public_subspace}
\end{algorithm}
\subsection{Algorithm Description}
\label{sec:algorithm_description}
Here we describe the noisy AdaGrad algorithm $\nadagrad$ presented in Algorithm~\ref{alg:noisy_adagrad_public_subspace}. It differs from the traditional AdaGrad in three ways: i) The pre-conditioner matrix at each stage is a noisy perturbation $H_t$ of the traditional pre-conditioner; ii) The state updates ($\theta_t\to\theta_{t+1}$) are dependent on noisy gradients, i.e., $\tilde{\nabla_t} = \nabla_t+b_t$ where $b_t$ represents the noise; iii) before applying on the gradients, the pre-conditioners ($H_t$'s) are  projected onto the rank $k_t$ subspace defined by $\tilde{V}_t$, the matrix returned by the subspace oracle.

\subsection{Regret Analysis}
\label{sec:regret_noisy_ada}

In this section we provide the regret analysis of noisy AdaGrad in Theorem~\ref{thm:noisy_ada_noisy_public_subspace}. One can interpret the regret as a composition of three terms: i) $O(\mathbf{Tr}(G_T)/T)$ which is the same as in the original AdaGrad algorithm; ii) a term that depends on the gradient noise, which as we mentioned earlier can be upper bounded by $O(\tr(G_T)/T)$ given a sensitivity oracle; and iii) a term $\gamma$ that bounds the error from a noisy projection obtained from the subspace oracle. 

The proof of Theorem~\ref{thm:noisy_ada_noisy_public_subspace} goes through a careful matrix perturbation analysis, that controls the perturbation of the subspace spanned by the non-noisy pre-conditioner $G_t$ at each time step $t\in[T]$. Recall that  $\lambda_{_{\min>0}}(G_t)$ denotes the smallest positive eigenvalue of $G_t$.

\begin{thm}
\label{thm:noisy_ada_noisy_public_subspace}
            Let $V_t$ be the orthogonal matrix whose column space is the tracked gradient subspace up to time $t$, and $\tilde{V_t}$ an approximation returned by an oracle. Let $\gamma$ be a bound on the subspaces' principal angle difference, i.e., $\|V_tV_t^T -\tilde{V}_t\tilde{V}_t^T  \|_{op} \leq \gamma$. Let $L$ be the gradient $\ell_2-$norm bound, $C$ the diameter of the constraint set $\calC$, and assume $L=C=O(1)$. Letting $\eta$ be the learning rate, $\sigma^2_b(t)$ be the gradient noise variance,  and choosing the pre-conditioner noise   such \linebreak that  $\sigma_B(t) \leq 2 \lambda_{min>0}(G_t)$, then running $\nadagrad$ on $\calL(\theta;D)$ for $T$ iterations we get 
\begin{align}
    \label{eq:noisy_ada_public_data}
    \mathbb{E}[ \regret{\calF}{\nadagrad}] \leq O \left(\mathbb{E}\left[\left(\frac{1}{\eta T} + \frac{\eta}{T}\right) \tr (G_T)  + \frac{\eta}{T}\sum_t \sigma_b^2(t)\tr (G_t^{-1}) + \gamma \right]\right)
\end{align}
\end{thm}

\mypar{Comparison with traditional Adagrad} 
We first introduce  a  definition that will allow us to determine conditions under which it is possible to achieve AdaGrad rates. 

\begin{definition}
\label{def:envelope}
Let $\calA$ be an optimization algorithm for solving  Problem~\ref{eq:problem_def}, that at time $t$ outputs result $\theta_t$. We define $L_{\calL, \calA}(t)$ as the function that asymptotically bounds from above and below the gradient norm $\max_i\| \nabla\ell_i(\theta_t) \|$ at iteration $t$ of algorithm $\calA$, i.e., $L_{f,\calA}(t) = \Theta(\max_i\| \nabla\ell_i(\theta_t) \|)$.
\end{definition}

We will drop the subindices $f, \calA$, since it will refer to our loss function $\calL$ and algorithm $\nadagrad$ in our paper. 

AdaGrad achieves a regret bound that is within a constant factor of $2C$ of the regret achieved by the best, fixed pre-conditioner in hindsight. We formalize this in Theorem \ref{thm:regret_adagrad}. Selecting the learning rate that minimizes the expression on the right of Equation~\ref{eq:noisy_ada_public_data} we obtain the result in the informal Theorem \ref{thm:noisy_ada_regret_informal}.

Assume constant rank $k_t=O(1)$ smaller than the problem dimension $p$. In the worst case, when $\sigma_b(t) = \Theta(1)$, these terms balance to $O(1/\sqrt T)$ and we obtain the same rates achieved by PDP-SGD. 
Assuming $\sigma_b(t) = L(t)$, these additional terms simplify to $O\left(\tr(G_T)/T + \gamma \right)$. That is, in this setting we recover AdaGrad rates:



\begin{cor}{(Appendix~\ref{proof:smooth_convergence})}
\label{cor:smooth_convergence}
Let $\sigma_b(t) = O(\|\nabla_t \|)$ in Algorithm~\ref{alg:noisy_adagrad_public_subspace}. With an appropriate learning rate, the overall regret of $\nadagrad$ is $O(\tr(G_T)/T + \gamma)$. Further, if  $\frac{1}{T} \sum_t \|\nabla_t \| = o(1)$ then $O(\tr(G_T)/T) = o(1/\sqrt{T})$.
\end{cor}

\begin{remark}
\label{remark:gradient_norm}
How to access $\|\nabla_t \|$ to design the noise is an open direction that we leave for future work. In practice we can find a bound on the expected norm schedule of the gradient, or rely on estimating it from public data, add adaptive gradient clipping to our algorithm according to this schedule (see for example \cite{TAB19,pichapati2019adaclip}), and design the noise according to these clipping values to obtain the desired rates.  
For example, if $\|\nabla_t \|$ is decreasing as $ O\left( \frac{1}{\sqrt[4]{T}}\right)$  then $\tr(G_T) = O (\sqrt{\sum_t 1/\sqrt[4]{t}}) =O(T^{3/8})$, and the regret decreases as $O(T^{-5/8})$, improving over SGD whose rates are in the order of  $O(1/\sqrt{T})$. 
\end{remark}

\subsection{Proof sketch}

Dimension independence is obtained thanks to the following observations: (1) the projection step given by $V_tV_t^T$ in \eref{eq:preconditioner_projection} allows us to work in a $k$ dimensional subspace instead of a $p$ dimensional one, and we may use Lemmas~\ref{lemma:operator_norm}, and  \ref{lemma:good_set} to bound the operator norm of $B_t$ restricted to this subspace in term of $k$ instead of $p$; (2) even if this projection ``erases" part of the real update, this error also lies in a $k$ dimensional subspace by assumption.

More concretely, the proof is structured as follows: paralleling traditional convergence proofs for descent algorithms, we will expand the expression $\| \theta_{t+1} - \theta^* \|_{H_t}$ to obtain an expression involving $ \langle\nabla_t, \theta_t -\theta^*\rangle$, and bound the regret using convexity.

Four terms are introduced that we will bound independently: two of them, one that depends on $\| \theta_{t+1} - \theta^* \|_{H_t} - \| \theta_{t} - \theta^* \|_{H_t}$ , and the norm of the gradients under $H^{-1}$, are analogous to the original AdaGrad proof; however, due to the noisy, projected pre-conditioner, we need structural Lemmas~\ref{lemma:operator_norm}, and \ref{lemma:good_set} to get around pre-conditioner noise $B_t$, and these terms can be finally bounded by $\mathbf{Tr}(G_t)$, up to a multiplicative factor. The connection is attained first by decomposing $\boldR^p$, isolating the subspace where $H_t$ is invertible. Thanks to the above observation (1) the relevant subspaces are $k$-dimensional. Then, restricted to this space, we rely on Lemma~\ref{lem:h-inverse}  that uses Woodbury identity to calculate the inverse of a sum of matrices (restricted $G_t$ and $B_t$ in this case), and Holder's inequality. The third term is the norm of $b_t$, the gradient noise that is similarly bounded using Lemma~\ref{lem:h-inverse}. Finally, we track the error introduced by the projection using Davis-Kahan theorem \citep{davis1963rotation} which again also lies in a $k$-dimensional subspace (observation (2)).
\section{Private Pre-conditioned Gradient Descent for ERM}
\label{sec:private_ada}
\subsection{Estimating subspace with public data}

In this section we will use Noisy-AdaGrad algorithm to define an $(\epsilon,\delta)$-differentially private algorithm $\apriv$ that approximately minimizes the excess empirical risk defined in~\eqref{eq:risk}. Our main contribution is in the low-rank unconstrained setting where, compared with original AdaGrad, we only pay an additional price of scale $\tilde{O}\left(\frac{1}{\epsilon n} \right)$, independent of dimension. To do so, we make the following observations:

\begin{itemize}
    \item {\bf Online to batch conversion:} If we set each of the loss function to be identical to $f_t(\theta)=\calL(\theta;D)$, and set $\privT=\frac{1}{T}\sum\limits_{t=1}^T \theta_t$ output by Algorithm~\ref{alg:noisy_adagrad_public_subspace} (Algorithm $\nadagrad$), then\\ $\mathbb{E}\left[\risk{\privT}\right]\leq \mathbb{E}\left[\regret{\calF}{\nadagrad}\right]$. (This follows from standard use of Jensen's inequality.)
    \item {\bf Computing $(\epsilon/2,\delta/2)$-private pre-conditioner :}  Lemma~\ref{lem:preconditioner_sensitivity} below and standard use of Renyi composition theorem~\citep{mironov2017renyi} imply that ensuring $\sigma_B(t)=O\left(\frac{L\sqrt{T t\log(1/\delta)}}{\epsilon n} \right)$ in Algorithm $\nadagrad$ ensures $\left(\frac{\epsilon}{2},\frac{\delta}{2}\right)$-differential privacy to the computation of all the $\boldH_t$'s in Algorithm $\nadagrad$.
    
    \item {\bf Ensuring all noisy gradients preserve $(\epsilon/2,\delta/2)$-differential privacy:} By the same argument as above, ensuring $\sigma_b(t)=O\left(\frac{L\sqrt{T\log(1/\delta)}}{\epsilon n}\right)$ in Algorithm $\nadagrad$ ensures $\left(\frac{\epsilon}{2},\frac{\delta}{2}\right)$-differential privacy to the computation of all the $(\nabla_t+b_t)$'s in Algorithm $\nadagrad$.
\end{itemize}

With these observations, and composition for $(\epsilon,\delta)$-differential privacy~\citep{dwork2014algorithmic} we can ensure the above variant of noisy AdaGrad is $(\epsilon,\delta)$-differentially private. In the following, we will use the online to batch conversation mentioned above to bound the excess empirical risk. In particular, we obtain a bound of $\tilde O(\frac{1}{\epsilon n})$ that does not depends on the dimensionality $p$. We formalize this result in the following corollary. In the setting where the pre-conditioner does not satisfy low-rank assumption, we will recover the traditional upper bound of $\tilde \Theta(\sqrt{p}/(\epsilon n))$ for private ERM via differentially private gradient descent~\citep{BST14}, since $\tr(G_t)$ will be growing with the dimension.

\begin{lem}{(Appendix~\ref{proof:preconditioner_sensitivity}) }\label{lem:preconditioner_sensitivity}
Let $G_t = \sqrt{\sum_t\nabla_t\nabla_t^T}$ be the preconditioner formed at iteration $t$. Let $\ell_i$ be an $L-$Lipschitz loss function on datapoint $d_i$ for $i=1,...,n$, and  $n$ the total number of records. Then the preconditioner's $\ell_2 - $sensitivity is given by $\Delta_2(G_t) = O\left( \frac{L\sqrt{ t}}{n} \right)$ 
\end{lem}

\begin{cor}{(Appendix~\ref{proof:private_ada})}
\label{cor:private_ada}
Assume the subspace spanned by accumulated gradients is bounded by a constant $k<p$. Let $\alpha$ be a non-negative real number such that $\|\nabla_t \| =O( \frac{1}{t^{\alpha}}) $. With appropriate choice of $\eta$, and for $\gamma = O(\frac{1}{\epsilon n})$, after $T = (\epsilon n )^{2/(1+2\alpha) }$, the excess risk of noisy-subspace $\nadagrad$ is $O(\frac{\sqrt{\log(1/\delta)}}{\epsilon n})$. 
\end{cor}

\begin{remark}
\label{remark:private_ada_constant_noise}
Connecting with the discussion in Remark~\ref{remark:gradient_norm},  Corollary~\ref{cor:private_ada}  does not assume noise $\sigma_b(t) =O( \| \nabla_t \|) $ since doing so would violate privacy. However, access to this quantity would give us enough information to find the sensitivity of the gradient $\nabla$ in certain settings and further improve rates. For example, under an interpolation assumption (see \cite{bassily2018exponential,ma2018power}), a bound on the norm of the average gradient $\| \nabla\calL(\theta_t;D) \|$ implies a bound on the envelope $L(t)$ of individual gradients $\|\nabla\ell_i(\theta_t)\|$ (Definition~\ref{def:envelope}).

For now, we rather assume constant noise, leaving us with suboptimal rates. However, we still reach the excess risk bound in fewer iterations than DP-SGD and PDP-SGD, in the case where  gradient norm is decreasing ($\alpha>0$ in Corollary~\ref{cor:private_ada} ): we require $T = (\epsilon n )^{2/(1+2\alpha) }$, compared to  $T = \epsilon^2n^2$ in (P)DP - SGD. 
\end{remark}


\subsection{Discussion: Privately Estimating the Subspace may not Help}

\begin{itemize}
    \item A natural way to avoid using public data is to privately estimate the subspace with differential privacy. Still, even if we estimate the subspace with 1/2 the data, there will be a dependence on $\gamma=\sqrt{p}/n$ by the best known upper bound. (See Theorem 2 in~\cite{dwork2014analyze}.) This is fundamental in the constrained optimization setting, where there exists a lower bound of $\Omega(\sqrt{p}/\epsilon n )$. 
    
    An open question that remains is if there exists a more direct analysis of private AdaGrad in the unconstrained setting that could achieve dimension independence without oracle access to the gradient subspace. In Section \ref{sec:unconstrained_sgd} we prove that this is possible for general convex functions in the unconstrained setting with only DP-SGD.
    
    \item Subspace estimation from public data: This problem has been widely explored in the literature, where a covariance matrix is to be estimated from $m$ (public) records sampled from  distribution $\calP$. More concretely, it is shown for example in \cite{zhou2020bypassing} that under natural assumptions $\gamma< O \left( \sqrt{\frac{\log{p}}{m}}\right)$ in the worst case scenario. 
\end{itemize}

\section{Interlude: Dimension independence in Unconstrained DP-SGD}
\label{sec:unconstrained_sgd}
Below we extend the results in \cite{song2020characterizing}, and show that unconstrained DP-SGD achieves dimension independence for general convex functions, without access to public data. 

\begin{thm}[Dimension independence in unconstrained optimization (Appendix~\ref{proof:unconstrained_dpgd})]
\label{thm:unconstrained_dpgd}
Let $\theta_0 = \bold0$ be the initial point of $\dpgd$. Let $\theta^* = \argmin\limits_{\theta \in \mathbb{R}^p}f(\theta) $ and $ M=VV^T $ be the projector to the gradients eigenspace. Letting $L$ be the gradient $\ell_2-$norm bound, setting the constraint set $\calC = \mathbb{R}^p$, and running $\dpgd$ on $\calL(\theta;D)$ for $T=\epsilon^2n^2$ and appropriate learning rate $\eta$, 
\begin{equation}
    \label{eq:unconstrained_dpgd}
    \mathbb{E}[\calL(\ptheta;D)]-\calL(\theta^*;D)\leq \frac{L\|\theta^*\|_{M}\sqrt{1+2\text{rank}(M)\log(1/\delta)}}{\epsilon n}
\end{equation}

\end{thm}

\section{Discussion}

We provide several insights that widen the understanding differentially private constrained and unconstrained optimization. First, with knowledge of the subspace where the gradients lie, it is possible to obtain bounds in terms of the trace of the pre-conditioner. This last one in turn, encodes the intrinsic dimension of the data, a smoother definition of the rank. Formally, the intrinsic dimension is defined for a positive-semidefinite matrix $A$ as the quantity 

$$ \text{intdim}(A) = \frac{\tr(A)}{\|A \|_{op}}$$

It measures the number of dimensions where $A$ has spectral content (see \cite{tropp2015introduction}). We can interpret our bound $\frac{\tr(G_T)}{T} = \frac{\text{intdim}(A)\|G_T\|_{op}}{T}$ as being dependent on the intrinsic dimension, rather than $p$, and the rate at which gradients are decreasing, captured by $\frac{\|G_T\|_{op}}{T}$

Second, we introduce the importance of a gradient norm schedule during the optimization is necessary to guarantee differential privacy without sacrificing running time. We leave it as a future direction the exploration of differentially private algorithms that provide access to this envelope.

\bibliographystyle{alpha}
\bibliography{reference}

\newpage
\appendix
\section{Additional background details.}

\subsection{AdaGrad algorithm}
\label{sec:appendix_adagrad}
Below we present an adaptation to our notation of the original AdaGrad algorithm from \cite{duchi2011adaptive}. 

\begin{algorithm}
\SetAlgoLined
\KwIn{Learning rate $\eta >0$, initial accumulator $\delta \geq 0$, bounded convex constraint set $\calX$}
$x_1 \gets \bold0$\;
$S_0 \gets \bold0$\;
$G_0 \gets \bold0$\;
$H_0 \gets \bold0$\;

 \For{t=1 \text{to} T}{
  Predict $\theta_t$, suffer loss $f(\theta_t)$ \;
  Update:
      \begin{align}
        \label{eq:preconditioner_update}  S_t & = S_{t-1} + \nabla_t\nabla_t^T, \quad G_t = S_t^{1/2} , \\
        H_t  & = \delta\boldI+\boldG_t
        \end{align}

    \begin{align}
           \label{eq:adagrad_update}   
   & \theta_{t+1} = \argmin_{\theta \in \calX} \| \theta - (\theta_t - \eta\nabla_t) \|^2_{H_t},
      \end{align}
 }
 \KwResult{$ \{ \theta_t \}$}
 \caption{Adagrad ($\calA_{\sf AdaGrad}$)}
 \label{alg:adagrad}
\end{algorithm}

\subsection{Differential Privacy}
\label{appendix:dp}

\begin{definition}
The Gaussian Orthogonal Ensamble (GOE) is the random matrix model of symmetric matrices $M_n$ where the upper triangular entries have distribution $\calN(0,1)$, and the diagonal entries $\calN(0,2)$. We use $\mu_{GOE}$ to denote the distribution of a matrix generated from this model. 
\label{def:GOE}. 
\end{definition}

\begin{definition}
\label{def:neighbor_datasets}
We say that two datasets $D$ and $D'$ are neighbors, and use the notation $D\sim D'$, if they differ in exactly one record, meaning there is exactly one record that is present in one and not in the other. 
\end{definition}

\begin{thm}[ Theorem A.1. in \cite{dwork2014algorithmic} ]
\label{thm:gaussian_mechanism}
Let $f$ be an arbitrary function with range in $\boldR^p$. Define its $\ell_2$-sensitivity as $\Delta_2(f) = \max_{D\sim D'}\|f(D) - f(D')\|_2$. Let $\epsilon \in (0,1)$, $c^2>2\ln(1.25\delta)$, and $\sigma \geq \frac{c\Delta_2(f)}{\epsilon}$. The Gaussian mechanism with parameter $\sigma$ that adds noise $\calN(0, \sigma^2)$ to all $p$ components is $(\epsilon, \delta)-$differentially private. 
\end{thm}

\mypar{Gaussian mechanism and strong composition} It follows from Renyi composition (Proposition 1 and 3 in \cite{mironov2017renyi}) that ensuring $\sigma = O \left(\frac{\Delta_2(f)\sqrt{T\log(1/\delta)}}{\epsilon}\right)$ preserves $(\epsilon, \delta)-$differential privacy when composing $T$ times the Gaussian mechanism with parameter $\sigma$. See Appendix~\ref{appendix:dp} for details and notation.

\section{Proof of Theorem~\ref{thm:noisy_ada_noisy_public_subspace}}

\subsection{Proof of Theorem~\ref{thm:noisy_ada_noisy_public_subspace}}

Below we present the detailed proof of Theorem~\ref{thm:noisy_ada_noisy_public_subspace}, and defer the proofs of structural Lemmas to Section~\ref{sec:other_proofs}. We first split the regret in four terms in Section~\ref{sec:preliminaries} and bound each of these independently.

\subsubsection{Preliminaries }
\label{sec:preliminaries}
Notice that $H_t$ may not be full rank, so we interpret $H_t^{-1}$ as the Moore-Penrose pseudoinverse for $t=1,...,T$. Let $A_t = \text{rowspace}(H_t)=\text{rowspace}(\Pi_{t}(G_t+B_t))$, so that $A_t^\perp = \text{ker}(H_t)$. We will decompose $\boldR^p$ into the following mutually orthogonal subspaces: 

\begin{align*}
    B_t &= \text{rowspace}(G_t)|_{A_t} \\
    C_t & = \text{rowspace}(G_t)|_{A_t^{\bot}}\\
    D_t &= \text{ker}(G_t).
\end{align*}

\begin{lem}{(Appendix~\ref{proof:preliminaries})}
\label{lemma:preliminaries}
Under the same assumptions of Theorem~\ref{thm:noisy_ada_noisy_public_subspace} 
\begin{align}
   \mathbb{E} \left[ f\left(\frac{1}{T} \sum_t \theta_t\right) - f(\theta^*) \right]& \leq \mathbb{E}_{b_1,..., b_{T-1},B_1,...,B_T} \left[
   \sum_{t=0}^T \frac{1}{2\eta T } \left(  \| \theta_t-\theta^*\|^2_{H_t} - \|\theta_{t+1}-\theta^* \|^2_{H_t} \right) \right.\\ \nonumber
     & \hspace{5mm} + \frac{\eta}{2T} \left(
      \mathbb{E}_{b_t}[ \|b_t\|_{H_t^{-1}}^2  | B_t]\ ]\right) \\
     \nonumber
     &  \hspace{5mm} + \frac{\eta}{2T} \left(\|\nabla_t \|_{H_t^{-1}}^2\right)\\
   \nonumber   & \hspace{5mm} + \frac{1}{T}\langle \nabla_t, \theta_t - \theta^* \rangle|_{C_t} \Bigg]
\end{align}

\end{lem}


We bound the four terms in this expression independently. 

\subsection{First Term: $ \| \theta_t-\theta^*\|^2_{H_t}  - \|\theta_{t+1}-\theta^* \|^2_{H_t}$} 

    \begin{align*}
    \sum_{t=0}^T (  \| \theta_t-\theta^*\|^2_{H_t} & - \|\theta_{t+1}-\theta^* \|^2_{H_t} ) \\
    & \leq \|\theta_0-\theta^*\|^2_{H_0} - \| \theta_{T+1} - \theta^*\|_{H_T} +
        \sum_{t=1}^{T-1}(\theta_t-\theta^*)(H_t - H_{t-1})(\theta_t-\theta^*)\\ 
        \end{align*}
        
        The first term on the right hand side is 0, since $H_0 = 0$ and the second one is non-positive thanks to the projection step, so we can bound this entire term as
        
        \begin{align*}
        &\leq \sum_{t=1}^{T} C_{\sf ada}^2\sigma_{max}(H_t - H_{t-1})\\ 
        & \leq \sum_{t}^{T} C_{\sf ada}^2 \mathbf{Tr}(H_t - H_{t-1}) \\ 
        \end{align*}
        By linearity of the trace and projections,
        \begin{align*}
        &= C_{\sf ada}^2 \sum_{t}^{T-1} \mathbf{Tr}(\Pi_t(G_t))-\mathbf{Tr}(\Pi_t(G_{t-1})) +
        \mathbf{Tr}(\Pi_t(B_t))-\mathbf{Tr}(\Pi_t(B_{t-1}))\\
        \text{Since $G_t-G_{t-1}$ is positive semi}&\text{-definite, } \\
       & \leq C_{\sf ada}^2\sum_{t=1}^{T-1} \mathbf{Tr}(G_t)- \mathbf{Tr}(G_{t-1})  +\mathbf{Tr}(\Pi_t(B_t)) - \mathbf{Tr}(\Pi_t(B_{t-1})) \\
       \end{align*}
       
      Now, $\mathbb{E}[\Pi_t(B_t)] = \mathbf{0}$, so taking expected value on both sides respect to $B_1, ..., B_T$, conditioned on $b_1, ..., b_{t-1}$, using linearity of expectation and independence of $b_t$-s and $B_t$-s
       
       \begin{align}
            \nonumber \mathbb{E} & \left[ \sum_t \left(  \| \theta_t -\theta^*\|^2_{H_t}  - \| \theta_{t+1}-\theta^* \|^2_{H_t} \right) \right] \\
  \nonumber & \leq \mathbb{E} \left[  C_{\sf ada}^2 \sum_{t=1}^{T-1} \mathbf{Tr}(G_t) - \mathbf{Tr}(G_{t-1})  +\mathbf{Tr}(\Pi_t(B_t)) -\mathbf{Tr}(\Pi_t(B_{t-1})) \right] \\
\nonumber & \leq \mathbb{E}  \left[C_{\sf ada}^2\sum_{t=1}^{T-1} \mathbf{Tr}(G_t) -\mathbf{Tr}(G_{t-1}) \right] \\
        & =  \mathbb{E} [ C_{\sf ada}^2\mathbf{Tr}(G_T)] 
       \end{align}

\subsection{Second term: $\mathbb{E}[\|b_t\|^2_{H_t^{-1}}]$}
\label{sec:norm-noise-bound}

For the third and fourth term we will first define subspaces where $H_t$ and $G_t$ can be inverted and then use the following lemma in these subspaces.

\begin{lem}{(Appendix~\ref{proof:h-inverse})}
\label{lem:h-inverse}
Define $C = A+B$, for $A,B,C$ linear operators on $\boldR^n$ such that $2 \|B\|_{op} \leq \lambda_{min>0}(A)$, and $B$ and $AB^{-1}+I$ invertible. Then for $v\in \text{im}(A) \cap \text{im}(C)$, $u \in \boldR^n$, $$\left| \langle u, C^{-1}v \rangle \right| \leq \frac{4}{3}\left|\langle u, A^{-1}v\rangle \right|$$
\end{lem}

    Now, to bound the norm of $b_t$ under $H_{t}^{-1}$. We proceed to decompose $\mathbb{R}^p$ into mutually orthogonal subspaces, concretely, into the mutually orthogonal row and null spaces of $G_t$.

    $$\mathbb{R}^p = \text{ker}(G_t)^{\perp}
    \bigoplus \text{ker}(G_t)$$
    
    Call these spaces respectively $A$, and $B$, we can write:
    $$\|b_t \|_{H_t^{-1}}^2 = \langle b_t, H_t^{-1}b_t\rangle$$
    $$= \langle b_t, H_t^{-1}|_A b_t\rangle+\langle b_t, H_t^{-1}|_B b_t\rangle  = \langle b_t, H_t^{-1}|_{A \cap C} b_t\rangle+\langle b_t, H_t^{-1}|_B b_t\rangle$$
        
    where $C = \text{im}(H_t)$. $B_t$ is full-rank with probability 1. As $G_t$ is the sum of projectors, its rowspace is contained in its image, so $A \subseteq \text{im}(G_t)$. Finally, $B_t$'s continuous eigenvalue distribution implies, for $v \in \boldR^n$,  $v \in \text{im}(G_tB_t^{-1}+I)$ with probability 1. Since $\boldR^n$ is finite-dimensional this yields the invertibility with probability 1.
    
    Further, every nonzero eigenvalue of $H_t$ is at least $\alpha(t) = 2\sqrt{k_t}\sigma_B$ by construction, and we may choose $c_1$ in the statement of the theorem appropriately such that $\|P^{k_t}_{H}B_t\|_{op} \leq \frac{\alpha(t)}{3}$ by Lemma \ref{lemma:good_set}. Consequently by Weyl's inequality the minimum nonzero eigenvalue of $\Pi_{\ker(H_t)^\perp}G_t$ must be at least $\frac{2\alpha(t)}{3}$.
    
    Therefore we can apply Lemma~\ref{lem:h-inverse} to $H_t, G_t$ and $B_t$. Using that $b_t$ is zero mean spherical noise with variance $\sigma_b^2(t)$ we obtain 
    
    \begin{align}
            \nonumber  \mathbb{E}_{b_t}[\|b_t\|^2_{H_t^{-1}}| b_1,\ldots, b_{t-1}, B_1,\ldots,B_T] & \leq \frac{4}{3}\mathbb{E}_{b_t}\left[b_t^TG_t^{-1}b_t\right]\\
 &=\frac{4\sigma_b^2(t) }{3}\mathbf{Tr}(G_t^{-1}) 
   \end{align}
   
   Where the inequality in the second step follows from the use independence and linearity of expectation: $$\mathbb{E}_{b_t}\left[b_t^TG_t^{-1}b_t\right] =  \mathbb{E}_{b_t}\left[\sum_{i,j}G^{-1}_{ij}b_{t_i}b_{t_j}\right] = \sum_{i,j}\mathbb{E}_{b_t}\left[G^{-1}_{ij}b_{t_i}b_{t_j}\right] = \sum_i\mathbb{E}_{b_t}\left[G^{-1}_{ii}b_{t_i}^2\right] = \sigma_b^2(t) \mathbf{Tr}(G_t^{-1})$$
   We claim that the composition of projections onto the top-$k_t$ eigenspace of $G_t + B_t$ and the kernel of $G_t$ is 0. This conclusion can alternately be stated as: no vector in the kernel of $G_t$ can be in the top-$k_t$ eigenspace of $B_t + G_t$.

    As $H_t = P_H^{k_t}B_t$ on ker$(G_t)$, this conclusion is implied by showing that $$\|P^{k_t}_{H_t}B_t\|_{op} < \lambda_{k_t}(H_t)$$

    but this is immediate by the conclusions of Lemma \ref{lemma:good_set}, again under appropriate choice of universal constant. Therefore $H_t$ is in fact the zero operator on subspace $B$, and does not contribute to the bound on $\|b_t\|_{H_t^{-1}}^2$.

Taking the sum over $t$, and expectation over remaining terms,



\begin{align}
\mathbb{E}\left [\sum_t \|b_t\|^2_{H_t^{-1}}\right] &\leq   \mathbb{E} \left[ \frac{4}{3} \sum^T_{t=1} \sigma_b^2(t)\tr(G_t^{-1})\right]
\end{align}

\subsection{Third term: $\|\nabla_t\|_{H_t^{-1}}$}

Paralleling the proof in the previous section, Section~\ref{sec:norm-noise-bound}, using the space decomposition and Lemma~\ref{lem:h-inverse}, we have that

    \begin{align}
        \|\nabla_t  \|_{H_t^{-1}}^2  &\leq \frac{4}{3} \nabla^T_tG_t^{-1}\nabla_t
    \end{align}

 Below we will bound this term using the following lemma. 
        
\begin{lem}[Lemma 5.15 in \cite{hazan2019introduction}] \label{lemma:gradients}
$$\sum_t\|\nabla_t\|^2_{G_t^{-1}} \leq 2 \mathbf{Tr}(G_T)$$
\end{lem}

Taking the sum over $t$, applying Lemma \ref{lemma:gradients}, and taking expectation over the conditioned terms, 

\begin{equation}
    \mathbb{E} \left[ \sum_t\|\nabla_t\|^2_{H_t^{-1}} \right]\leq \mathbb{E} \left[   \frac{4}{3} \cdot 2 \cdot  \mathbf{Tr}(G_T) \right]
\end{equation}

    
\subsection{ Fourth term: $\langle \nabla_t, \theta_t - \theta^* \rangle|_{C_t}$}

This term corresponds to the component of $\nabla_t$ we could have lost in the projection due to an innacurate gradient subspace estimation.   


\begin{lem}{(Appendix~\ref{proof:good_set} )}
\label{lemma:good_set}
For $H_t$, $G_t$, $B_t$ as in the statement of the theorem, assuming $\lambda_{_{\min>0}}(G_t) > \alpha(t)$ for $t > t_0$, and $\text{rank}(G_t)\geq 1$, for any $\eta > 0$ there is some universal $c$ such that the event
$$E = \bigcap_{t=t_0}^T \left(E_t  \cap F_t\right)= \bigcap_{t=t_0}^T \left\{ \|P^{(j_t)}_{G_t}B_t\|_{op} \leq c\sqrt{\text{rank}(G_t)}\log(t)\right\} \cap \left\{ \|B_t\|_{op} \leq c\sqrt{p}\log(t)\right\}$$
satisfies
$$\mathbf{Pr}(E) > 1-\eta.$$
\end{lem}


   Recall that $C_t$ corresponds to the intersection of the accumulated gradient subspace (row space of $G_t$) with the kernel of $H_t$. So this term can be expanded as follows, using Cauchy-Schwartz for the first inequality and Davis-Kahan theorem for the second one (see Appendix~\ref{thm:davis-kahan}). 
   
   \begin{align}
       \nonumber    \langle \nabla_t, \theta_t - \theta^* \rangle|_{C_t} &= \nabla_t^T P_{G_t}^{j_t}(I-P_{H_t}^{k_t})(\theta_t - \theta^*)\\
       \nonumber &\leq \|\nabla_t\|_2\|P_{G_T}^{j_t}(I-P_{H_t})(\theta_t - \theta^*)\|_2 \\
       \nonumber &\leq CL\|V_tV_t^T - \tilde{V}_t\tilde{V_t}^{T}\|_{op} \\
       & \leq CL\gamma
   \end{align}
   

 



\subsection{Putting these estimates together}

Finally, putting together the four expressions,

\begin{align}
\label{lr_dependence}
\nonumber
     f(\frac{1}{T} \sum_t \theta_t) - f(\theta^*) & \leq\mathbb{E} \left[ \left( \frac{C_{\sf ada}^2}{ 2\eta T}+ \frac{8\eta}{3 \cdot2T} \right) \mathbf{Tr}(G_T) \right]\\
     \nonumber
     & + \frac{\eta}{2T} \mathbb{E}\left[ \sum_{t} \sigma_b^2(t)\tr(G_t^{-1})\right] \\
    &+ \gamma
\end{align}

We obtain the informal version in Theorem~\ref{thm:noisy_ada_regret_informal} by picking the minimizing learning rate. 

\section{Proof of structural lemmas and Theorem~\ref{thm:unconstrained_dpgd}}
\label{sec:other_proofs}

\subsection{Lemma~\ref{lemma:preliminaries}}
\label{proof:preliminaries}
\textit{Under the same assumptions of Theorem~\ref{thm:noisy_ada_noisy_public_subspace} 
\begin{align}
   \mathbb{E} \left[ f\left(\frac{1}{T} \sum_t \theta_t\right) - f(\theta^*) \right]& \leq \mathbb{E}_{b_1,..., b_{T-1},B_1,...,B_T} \left[
   \sum_{t=0}^T \frac{1}{2\eta T } \left(  \| \theta_t-\theta^*\|^2_{H_t} - \|\theta_{t+1}-\theta^* \|^2_{H_t} \right) \right.\\ \nonumber
     & \hspace{5mm} + \frac{\eta}{2T} \left(
      \mathbb{E}_{b_t}[ \|b_t\|_{H_t^{-1}}^2  | B_t]\ ]\right) \\
     \nonumber
     &  \hspace{5mm} + \frac{\eta}{2T} \left(\|\nabla_t \|_{H_t^{-1}}^2\right)\\
      & \hspace{5mm} + \frac{1}{T}\langle \nabla_t, \theta_t - \theta^* \rangle|_{C_t} \Bigg]
\end{align}
}
\begin{proof}
Recall that given $H$, we define the scalar product $\langle x,y\rangle_{H}:=\langle x, H y\rangle$, and we use the notation $\cdot |_{A}$ to denote the output of a transformation restricted to subspace $A$. \\

 Following the update rule in \eref{eq:adagrad_noisy_update}, 
 
\begin{align}
   \|\theta_{t+1}-\theta^* \|^2_{H_t} &=  \|P_{\calC}^{H_t}(\theta_{t}-\eta H_t^{-1}(\nabla_t+b_t)) - \theta^* \|^2_{H_t}  \\ 
   \label{eq:free_projection}  & \leq  \|\theta_{t}-\eta H_t^{-1}(\nabla_t+b_t) - \theta^* \|^2_{H_t}\\
  \nonumber  &=  \|\theta_{t}-\theta^* \|^2_{H_t} - 2\eta \langle H_t^{-1}(\nabla_t + b_t), \theta_t - \theta^* \rangle_{H_t} + \eta^2\|\nabla_t+b_t\|_{H^{-1}_t}^2.
\end{align}

Where the first steps follows by the contraction property of projections (see Appendix~\ref{lemma:projection_contraction})


We have that 
\begin{equation}
\label{eq:preliminaries}
    \langle H_t^{-1}(\nabla_t + b_t), \theta_t - \theta^* \rangle_{H_t} = \langle \nabla_t + b_t, \theta_t - \theta^* \rangle|_{A_t}.
\end{equation}

Rearranging, 

\begin{equation}
    \langle \nabla_t + b_t , \theta_t - \theta^* \rangle|_{A_t} 
= \frac{1}{2\eta } \left( \| \theta_t-\theta^*\|^2_{H_t} - \|\theta_{t+1}-\theta^* \|^2_{H_t} \right) + \frac{\eta}{2} \|\nabla_t + b_t\|_{H_t^{-1}}^2
\end{equation}

Taking conditional expectation over $b_t$, conditioned on $b_1, ...b_{t-1}, B_1, ..., B_{t}$ the left hand side becomes $$\mathbb{E}_{b_t}[\langle \nabla_t + b_t  , \theta_t - \theta^* \rangle|_{A_t} | b_1, ..., b_{t-1}, B_1, ..., B_{t}] =\langle \nabla_t, \theta_t - \theta^* \rangle|_{A_t} $$ 

Traditionally, we could now use convexity to bound the regret by using the identity $h(\theta_t) - h(\theta^*) \leq \langle \nabla_t  , \theta_t - \theta^* \rangle$. Notice though that we could have lost some signal after the projection step, and $\langle \nabla_t  , \theta_t - \theta^* \rangle \neq \langle \nabla_t  , \theta_t - \theta^* \rangle |_{A_t}$

However, we know that $$\langle \nabla_t, \theta_t - \theta^* \rangle = \langle \nabla_t, \theta_t - \theta^* \rangle|_{B_t} + \langle \nabla_t, \theta_t - \theta^* \rangle|_{C_t} + \langle \nabla_t, \theta_t - \theta^* \rangle |_{D_t}.$$ 

 Furthermore, by construction $\nabla_t \in \text{rowspace}(G_t)$ and thus: i) its product will be zero on $D_t$ and ii) we can interchange $B_t$ and $A_t$ since $B_t \subseteq A_t$, then

 \begin{align*}
 \langle \nabla_t, \theta_t - \theta^* \rangle & \leq \langle \nabla_t, \theta_t - \theta^* \rangle|_{A_t} + \langle \nabla_t, \theta_t - \theta^* \rangle|_{C_t} \\
 \end{align*}
 
 
 Completing this in Equation~\ref{eq:preliminaries}, and using the fact that  $b_t$-s are independent, we obtain

\begin{align}
   \nonumber  \langle \nabla _t  , \theta_t - \theta^* \rangle  &\leq \langle \nabla_t, \theta_t - \theta^* \rangle|_{C_t} \\
 \nonumber  & \hspace{5mm} +  \frac{1}{2\eta } \left( \| \theta_t-\theta^*\|^2_{H_t} - \|\theta_{t+1}-\theta^* \|^2_{H_t} \right)\\ 
     & \hspace{5mm} + \frac{\eta}{2} \left( \|\nabla_t \|_{H_t^{-1}}^2 +  \mathbb{E}_{b_t}[\|b_t\|_{H_t^{-1}}^2] \right)
\end{align}

Now we can invoke convexity, $f(\theta_t) - f(\theta^*) \leq \langle \nabla_t  , \theta_t - \theta^* \rangle$ and $f(\sum_t \theta_t) \leq \sum_t f(\theta_t)$.

Combining these facts and taking the sum over $t$, 

\begin{align}
  \nonumber   f\left(\frac{1}{T} \sum_t \theta_t\right) - f(\theta^*)& \leq \sum_t \frac{1}{T} \langle \nabla_t, \theta_t - \theta^* \rangle|_{C_t}\\
   & \hspace{5mm} +    \frac{1}{2\eta T } \left(  \| \theta_t-\theta^*\|^2_{H_t} - \|\theta_{t+1}-\theta^* \|^2_{H_t} \right) \\ \nonumber
     & \hspace{5mm} + \frac{\eta}{2T} \left(
      \mathbb{E}_{b_t}[ \|b_t\|_{H_t^{-1}}^2 | B_t]\right) \\
     \nonumber
     & \hspace{5mm} + \frac{\eta}{2T} \left(\|\nabla_t \|_{H_t^{-1}}^2\right)
\end{align}

Using law of total expectation,

\begin{align}
   \mathbb{E} \left[ f\left(\frac{1}{T} \sum_t \theta_t\right) - f(\theta^*) \right]& \leq \mathbb{E}_{b_1,..., b_{T-1},B_1,...,B_T} \left[
   \sum_{t=0}^T \frac{1}{2\eta T } \left(  \| \theta_t-\theta^*\|^2_{H_t} - \|\theta_{t+1}-\theta^* \|^2_{H_t} \right) \right.\\ \nonumber
     & \hspace{5mm} + \frac{\eta}{2T} \left(
      \mathbb{E}_{b_t}[ \|b_t\|_{H_t^{-1}}^2  | B_t]\ ]\right) \\
     \nonumber
     &  \hspace{5mm} + \frac{\eta}{2T} \left(\|\nabla_t \|_{H_t^{-1}}^2\right)\\
      & \hspace{5mm} + \frac{1}{T}\langle \nabla_t, \theta_t - \theta^* \rangle|_{C_t} \Bigg]
\end{align}

\end{proof}

\subsection{Contraction property of projection for arbitrary norms}
\label{lemma:projection_contraction}

\begin{lem}

Let $\|\cdot \|$ define a seminorm. Let $\calC$ be a convex set and $\Pi_{\calC}(x)= \argmin_{v\in \calC} \| v-x\|$ be the projection operator to set $\calC$. Then for any $v \in \calC$, 

$$\| \Pi_{\calC}(x) - v\| \leq \|x-v\|$$
\end{lem}

\begin{proof}
Notice the contraction and the projection are measured using the same seminorm. 

Let $x^* = \Pi_{\calC}(x)$, and $v\in \calC$, $v\neq x^*$

We first prove that 
\begin{equation}
    \label{eq:negative_dot_product}
    \langle x-x^*,v-x^* \rangle \leq 0 
\end{equation}
 Let $\alpha \in (0,1)$, then by convexity of $\calC$, $x^*+\alpha(v-x^*) \in \calC$, so by optimality of $x^*$
\begin{align*}
    \|x-x^* \|^2 & \leq \|x-(x^*+\alpha(v-x^*))\|^2 \\
    & = \|x-x^* \|^2 + \alpha^2\|v-x^* \|^2-2\alpha \langle x-x^*, v-x^2 \rangle\\
    \iff \langle x-x^*, v-x^2 \rangle & \leq \frac{\alpha}{2}\|v-x^* \|^2
\end{align*}
Which is true for $\alpha$ arbitrarily small, yielding Equation~\ref{eq:negative_dot_product}

Now, we start going backwards, 
\begin{align*}
     \|x -v\|^2  &\geq \| x^* - v\|^2\\
     &= \|x^* - x + x-v \| ^2\\
    &= \| x^* - x\|^2 +\| x-v\|^2 + 2\langle x^*-x,x-v \rangle\\
    & = \| x^* - x\|^2 +\| x-v\|^2 + 2\langle x-x^*,v-x^*+x^*-x \rangle\\
    &= -\| x^* - x\|^2 +\| x-v\|^2+ 2\langle x-x^*,v-x^* \rangle
\end{align*}
Cancelling terms, and rearranging, 

\begin{equation*}
    \iff \| x^* - x\| \geq 2\langle x-x^*,v-x^* \rangle
\end{equation*}

which is true by non-negativity of semi-norms and equation \ref{eq:negative_dot_product}

\end{proof}

\subsection{Lemma~\ref{lemma:good_set} }
\label{proof:good_set}
\textit{For $H_t$, $G_t$, $B_t$ as in the statement of the theorem, assuming $\lambda_{_{\min>0}}(G_t) > \alpha(t)$ for $t > t_0$, and $\text{rank}(G_t)\geq 1$, for any $\eta > 0$ there is some universal $c$ such that the event
$$E = \bigcap_{t=t_0}^T \left(E_t  \cap F_t\right)= \bigcap_{t=t_0}^T \left\{ \|P^{(j_t)}_{G_t}B_t\|_{op} \leq c\sqrt{\text{rank}(G_t)}\log(t)\right\} \cap \left\{ \|B_t\|_{op} \leq c\sqrt{p}\log(t)\right\}$$
satisfies
$$\mathbf{Pr}(E) > 1-\eta.$$
}
\begin{proof}
The following lemma is used to ensure the spectrum of $B_t$ is bounded with high probability.
\begin{lem}[Corollary 2.3.5 in \cite{tao2012topics}]
\label{lemma:operator_norm}
Suppose that the coefficients of matrix $M \in \mathbb{R}^{p\times p}$ are independent, have zero mean and uniformly bounded by 1. Then there exists absolute constants $C,c >0$ such that for all $A \geq C$

\begin{equation*}
    \Pr( \| M \|_{op} > A\sqrt{p}) \leq C\exp(-cAp)
\end{equation*}
\end{lem}

We will consider the complement of these sets $E_t$, and show that their probabilities sum up to some small constant. Let $\eta > 0$.

Begin by noting that Lemma~\ref{lemma:operator_norm} immediately implies
$$\Pr\left[ \|B_t\|_{op} \geq A\sqrt{p}\log(t)\right] \leq C \exp\left(-cA\log(t)p\right)$$

and therefore an appropriate choice of $A$ can be made such that
$$\sum_{t=0}^T \Pr\left[ \|B_t\|_{op} \geq A\sqrt{p}\log(t)\right] \leq \eta.$$

Now, $P^{(j_t)}_{G_t}$ can be written as multiplication by $V\Sigma V^T$ for $V$ orthogonal, $\Sigma$ diagonal matrix of 1s and 0s associated to the appropriate eigenvalues. Since the GOE is invariant under orthogonal conjugations, this implies that the distribution of $P^{(j_t)}_{G_t}B_t$ is identical to the GOE distribution on matrices of $\text{rank}(C_t) \leq \text{rank}(G_t)$. That is, for $\simeq$ denoting distributional equality,

$$P^{(j_t)}_{G_t}B_t = V\Sigma V^T B_t \simeq V\Sigma V^T V B_t V^T = V\Sigma B_t V^T \simeq V M V^T$$
where $M$ is the GOE over rank$(\Sigma) \times $rank$(\Sigma)$ matrices.

Therefore applying Lemma \ref{lemma:operator_norm} again to this lower-dimensional GOE, we obtain
$$\Pr\left[ \|P^{(j_t)}_{G_t}B_t\|_{op} \geq A\sqrt{\text{rank}(G_t)}\log(t)\right] \leq C \exp\left(-cA\log(t)\text{rank}(G_t)\right) \leq C \exp\left(-cA \log(t)\right)$$

and again an appropriate choice of $A$ implies

$$\Pr\left[ \|B_t|_{C_t}\|_{op} \geq A\sqrt{\text{rank}(G_t)}\log(t)\right] \leq \eta$$

and the conclusion follows.
\end{proof}







\subsection{Davis-Kahan Theorem}
\label{thm:davis-kahan}
\begin{thm}[Davis-Kahan Theorem]
            For any matrices $A$ and $B$ of like dimensions, for which $\lambda_i(A) > \lambda_j(B)$, 
               
               \begin{equation}
                \label{eq:david-kahan}
                \|P_A^i(I-P_B^{j-1})\|_{op} \leq \frac{\| A-B\|_{op}}{\lambda_i(A) - \lambda_j(B)}
            \end{equation}
        \end{thm}
\subsection{Lemma~\ref{lem:h-inverse}}
\label{proof:h-inverse}
\textit{Define $C = A+B$, for $A,B,C$ linear operators on $\boldR^n$ such that $2 \|B\|_{op} \leq \lambda_{min}(A)$, and $B$ and $(AB^{-1}+I)$ invertible. Then for $v\in \text{im}(A) \cap \text{im}(C)$, $u \in \boldR^n$, $$\left|\langle u, C^{-1}v \rangle \right|\leq \frac{4}{3}\left|u^TA^{-1}v\right|$$}

\begin{proof}
Since $A$, $B$, $A + B$ and $AB^{-1}+I$ can be inverted for $v \in \text{im}(A) \cap \text{im}(C)$, we can use the Woodbury identity (in its special case as Hua's identity, which does not rely on global but rather pointwise invertibility on its intermediate terms), to calculate $C^{-1} = (A +B)^{-1}$. We additionally use invertibility of $AB^{-1} + I$ to invert the order of the Moore-Penrose pseudoinversion on the product $(I + AB^{-1})A$ (see Corollary 1.4.1 of \cite{campbellmeyerinverses}).
            
        \begin{align}
            \nonumber \left|\langle u, C^{-1}v \rangle \right|&= \left|u^TC^{-1}v \right|\\
            \nonumber &= \left|u^T(A^{-1}-(AB^{-1}A + A)^{-1})v\right|\\
            \nonumber &= \left|v^TA^{-1}v - u^T (AB^{-1}A + A)^{-1})v \right|\\
            \nonumber & \leq  \left|u^TA^{-1}v\right| + |\mathbf{Tr}(u^T A^{-1}(B^{-1}A + I)P_{\text{im}(A)})^{-1}v) |\\
        \end{align}
Where the last step follows by the triangle inequality, and because the trace of a scalar is just that scalar.
Using the cyclic property of the trace,
            \begin{align}
            \nonumber &=\left|u^TA^{-1}v\right| + |\mathbf{Tr}(uv^T A^{-1}((AB^{-1} + I)P_{\text{im}(A)})^{-1})|\\
            \end{align}
Using the trace duality property, 
            \begin{align}
     \label{eq:holder} &\leq \left|u^TA^{-1}v\right| + \| u^TA^{-1}v\|_1 \|(AB^{-1} + I)^{-1})\|_{\infty}\\
            & \leq \left|u^TA^{-1}v\right| \left(1+\max_{i,j} \frac{\lambda_i(B)}{\lambda_i(B)+\lambda_j(A)}\right) 
     \end{align}

    We have that  $\lambda_j(A) \geq \lambda_{n}(A)$.  Since $2 \|B\|_{op} \leq  \lambda_j(A)$ for all $j$, the term on the right is maximized when $\lambda_{j}(A) = \lambda_{n}(A)$, and $\lambda_i(B) = \frac{\lambda_{n}(A)}{2}$
     \begin{align}
    \nonumber   &  \leq v^TA^{-1}v \left(1+\frac{\lambda_{n}(A)/2}{\lambda_n(A)/2+\lambda_{n}(A)} \right) \\
   \nonumber     & \leq v^TA^{-1}v \left(1+1/3\right) \\
     &\leq \frac{4}{3} v^TA^{-1}v
        \end{align}
        
\textbf{Remark: } Notice that in step \ref{eq:holder}, the $L_1$ and $L_{\infty}$ norms can be replaced by any $p$ and $q$ such than $\frac{1}{p} + \frac{1}{q} = 1$ to obtain a better bound. 

\end{proof}

\subsection{Corollary~\ref{cor:smooth_convergence}}
\label{proof:smooth_convergence}
\textit{Assume the norm of gradients is decreasing as $L(t)= o(1)$, and constant rank for the gradient subspace. Let $\sigma_b(t) = O(L(t))$, then the overall regret of $\nadagrad$ is $O(\tr(G_T)/T) =o(1/\sqrt{T})$. 
}

\begin{proof}
We first introduce the following inequality that has been previously used in optimization, see Lemma 1 in \cite{streeter2010less} for a proof. 

\begin{lem}
\label{lemma:inequality-sum-non-negative}
For any non-negative real numbers $a_1,a_2,a_3, \cdots , n, $, 

$$ \sum_{i=1}^n \frac{a_i}{\sqrt{\sum_j=1^i}} \leq2\sqrt{\sum_{i=1}^n a_i}$$
\end{lem}
Notice that $\tr(G_t) = O(\sqrt{\sum_{s=1}^t L^2(s)})$. Then selecting the optimal learning rate in expression \ref{eq:noisy_ada_public_data}, we have:

\begin{align}
    \nonumber \mathbb{E}[ \regret{\calF}{\nadagrad}] &\leq O \left(\mathbb{E}\left[ \frac{\sqrt{\tr(G_T)^2 + \tr(G_T)\sum_t\sigma^2_b(t) \tr(G_t^{-1})}}{T} + \gamma\right] \right)\\
  \nonumber  & \leq O \left(\mathbb{E}\left[ \frac{\sqrt{\tr(G_T)^2 + \tr(G_T)\sum_tL^2(t)/\sqrt{\sum_{s=1}^t L^2(s)}}}{T} + \gamma\right] \right)\\ 
  \label{eq:l1-l2norm}  & \leq O \left(\mathbb{E}\left[ \frac{\sqrt{\tr(G_T)^2 + \tr(G_T)\sqrt{\sum_tL^2(t)}}}{T} + \gamma\right] \right)\\ 
  \nonumber  & \leq O \left(\mathbb{E}\left[ \frac{\sqrt{\tr(G_T)^2 + \tr^2(G_T)}}{T} + \gamma\right] \right)\\ 
   \nonumber & \leq O \left(\mathbb{E}\left[ \frac{\tr(G_T)}{T} + \gamma\right] \right)\\ 
\end{align}

Where Equation~\ref{eq:l1-l2norm} follows from the non-negativity of $L^2(t)$ and inequality in Lemma~\ref{lemma:inequality-sum-non-negative}
\end{proof}

\subsection{Lemma~\ref{lem:preconditioner_sensitivity} - pre-conditioner sensitivity}
\textit{Let $G_t = \sqrt{\sum_t\nabla_t\nabla_t^T}$ be the preconditioner formed at iteration $t$. Let $\ell_i$ be an $L-$Lipschitz loss function on datapoint $d_i$ for $i=1,...,n$, and  $n$ the total number of records. Then the preconditioner's $\ell_2 - $sensitivity is given by $\Delta_2(G_t) = O\left( \sqrt{ \frac{tL}{n}} \right)$ }
\begin{proof}
\label{proof:preconditioner_sensitivity}
Let $G_t^{D}$ be the preconditioner computed at iteration $t$ with dataset $D = \{d_1, ..., d_n \}$. Let $D'$ be a neighboring dataset, w.l.o.g. $d_n \notin D'$.

Let $g_{t,j} =\frac{\nabla\ell(\theta_t;d_j)}{n}$. By the $L-$Lipschitz condition, $\|g_{i,j}\| \leq \frac{L}{n}$ Recall that $G^D_t = \sqrt{\sum_{i=1}^t\nabla_i\nabla_i^T}$, and $\nabla^D_i = \sum_{j \in D}g_{i,j}$.   Let $K_t = \sum_{i=1}^t (\sum_{j=1}^{n-1}g_{i,j})(\sum_{j=1}^{n-1}g_{i,j})^T$

We have then

\begin{align}
  \nonumber  \|G_t^D - G_t^{D'} \|_2 & = \left\lVert\sqrt{K_t + \sum_{i=1}^t{\sum_{j=1}^ng_{i,n}g_{i,j}}} - \sqrt{K_t} \right\rVert_F\\ 
  \nonumber& \leq \left\lVert \sqrt{\sum_{i=1}^t{\sum_{j=1}^ng_{i,n}g_{i,j}}} \right\rVert_F\\
  & = \sqrt{ \tr\left( \sqrt{\sum_{i=1}^t{\sum_{j=1}^ng_{i,n}g^T_{i,j}}} \sqrt{\sum_{i=1}^t{\sum_{j=1}^ng_{i,n}g^T_{i,j}}}^T\right) } \\
  &= \sqrt{\tr \left( \sum_{i=1}^t\sum_{j=1}^ng_{i,n}g_{i,j}^T \right)} \\
  & \leq \frac{L\sqrt{t}}{n}
\end{align}

\end{proof}

\subsection{Corollary~\ref{cor:private_ada}}
\label{proof:private_ada}
\textit{
Assume the subspace spanned by accumulated gradients is bounded by a constant $k<p$. With appropriate choice of $\eta$, and for $\gamma = O(\frac{1}{\epsilon n})$, the excess risk of noisy-subspace $\nadagrad$ is $O(\frac{\sqrt{\log(1/\delta)}}{\epsilon n})$. 
}

\begin{proof}
Recall that $\sigma_b(t) = O(\frac{L\sqrt{T\log(1/\delta)}}{\epsilon n})$, assume $L=O(1)$, and assume gradients norms are decreasing as $\|\nabla_t\|=O(\frac{1}{t^{\alpha}})$, then $\tr(G_T) = O(\sqrt{\sum_t 1/t^{2\alpha}})=O(T^{\frac{1-2\alpha}{2}})$ and $\sum_t\tr(G^{-1}_t) = O\left(\sum_t \frac{1}{t^{\frac{1-2\alpha}{2}}}\right) = O(T^{\frac{1+2\alpha}{2}})$. 

Replacing these values in Theorem~\ref{thm:noisy_ada_noisy_public_subspace}, 

\begin{align*}
    \mathbb{E}\left[\risk{\privT}\right]\leq \sqrt{\frac{T^{1-2\alpha}}{T^2} + \frac{T^{\frac{1-2\alpha}{2}} T T^{(1+2\alpha)/2}}{\epsilon^2n^2T}} + \gamma \\
    &\leq \sqrt{\frac{1}{T^{1+2\alpha}} + \frac{1}{\epsilon^2n^2}} + \frac{1}{\epsilon n}
\end{align*}

Then letting $T = \epsilon n)^{2/(1+2\alpha)}$ we obtain the desired result. 
\end{proof}

\subsection{Unconstrained DP-GD}
\label{proof:unconstrained_dpgd}
\begin{algorithm}
\SetAlgoLined
\KwIn{noise variance $\sigma^2$, number of iterations $T$, learning rate $\eta$, gradient oracle $\nabla_t$}
 \For{t=1 \text{ to } T}{
  $\tilde{\nabla}_t \gets \nabla_t + b_t \quad \text{ where } b_t \sim \mathcal{N}(0,\sigma_b^2 I_p)$\;
 $ \theta_{t+1}   =\theta_{t} - \eta\tilde{\nabla}_t$
 }
 \KwResult{$\frac{1}{T} \sum_{t=1}^T \theta_t$}
 \caption{DP-GD: Differentially private gradient descent}
 \label{alg:DP-GD}
\end{algorithm}
\mypar{Theorem~\ref{thm:unconstrained_dpgd}}
\textit{
Let $\theta_0 = \bold0$ be the initial point of $\dpgd$. Let $\theta^* = \argmin\limits_{\theta \in \mathbb{R}^p}f(\theta) $ and $ M=VV^T $ be the projector to the gradients eigenspace. Letting $L$ be the gradient $\ell_2-$norm bound, setting the constraint set $\calC = \mathbb{R}^p$, and running $\dpgd$ on $\calL(\theta;D)$ for $T=\epsilon^2n^2$ and appropriate learning rate $\eta$, 
\begin{equation*}
    \mathbb{E}[\calL(\ptheta;D)]-\calL(\theta^*;D)\leq \frac{L\|\theta^*\|_{M}\sqrt{1+2\text{rank}(M)\log(1/\delta)}}{\epsilon n}
\end{equation*}
}
\begin{proof}
We follow standard arguments for analyzing gradient descent \cite{song2020characterizing,bubeck2015convex} . 
Recall that $M$ is the projector to gradients eigenspace. 

We have that 
\begin{align*}
   \| \theta_{t+1} - \theta^*\|^2_M &= \| \theta_{t} - \theta^* - \eta(\nabla_t + b_t)\|^2_M \\  
   & \leq  \| \theta_{t} - \theta^*\|^2_M -2\eta \langle\nabla_t + b_t, \theta_t - \theta^* \rangle + \eta^2\|\nabla_t + b_t\|_M^2 \\
\end{align*}

Taking expected value respect to $b_t$ conditioned on $b_1, ..., b_{t-1}$, 

\begin{align*}
    &\leq  \| \theta_{t} - \theta^*\|^2_M -2\eta \langle\nabla_t, \theta_t - \theta^* \rangle + \eta^2(L^2 + \rank(M) \sigma^2) 
\end{align*}
Here we used that $\nabla_t$ lies in the subspace $M$, and $b_t \sim \calN(0, \sigma^2I_p)$

Rearranging,and taking expectation over $b_1, ..., b_{t-1}$

\begin{equation}
\label{eq:eq_sgd}
   \mathbb{E}\left[ \langle\nabla_t, \theta_t - \theta^* \rangle \right] \leq \frac{1}{2\eta}(\mathbb{E}\left[\| \theta_{t} - \theta^*\|^2_M -\| \theta_{t+1} - \theta^*\|^2_M \right]) + \frac{\eta}{2}(L^2 + \rank(M) \sigma^2) 
\end{equation}

By convexity, 

\begin{equation*}
    \calL(\privT;D) - \calL(\theta^*;D) \leq \frac{1}{T} \sum_{t=1}^T\langle\nabla_t, \theta_t-\theta^*\rangle
\end{equation*}

So taking the sum over $t$ in Equation~\ref{eq:eq_sgd} and using linearity of expectation we get

\begin{align*}
    \mathbb{E}\left[\calL(\privT;D)\right] - \calL(\theta^*;D) \leq \frac{1}{2\eta T}\| \theta_{0} - \theta^*\|^2_M + \frac{\eta}{2}(L^2 + \rank(M) \sigma^2) 
\end{align*}

Taking the optimal learning rate $\eta$, 

\begin{align*}
    &\leq \|\theta^*\|_M\sqrt{\frac{L^2 + rank(M)\sigma^2}{T}} 
\end{align*}

Setting $\sigma = O(\frac{L\sqrt{T\log(1/\delta)}}{\epsilon n})$, and $T=\epsilon^2 n^2$ we obtain the desired result:

\begin{equation*}
    \mathbb{E}[\calL(\ptheta;D)]-\calL(\theta^*;D)\leq \frac{L\|\theta^*\|_{M}\sqrt{1+2\text{rank}(M)\log(1/\delta)}}{\epsilon n}
\end{equation*}

\end{proof}

\end{document}